\documentclass[11pt]{article}

\usepackage{dsfont}
\usepackage[colorlinks, linkcolor=magenta, anchorcolor=cyan, citecolor=blue]{hyperref}

\usepackage{amsthm,amsfonts,amsmath,amssymb,epsfig,color,float,graphicx,verbatim, enumitem}
\usepackage{multirow}
\usepackage{algorithm}
\usepackage[noend]{algpseudocode}
\usepackage{enumitem}

\newif\ifhyper\IfFileExists{hyperref.sty}{\hypertrue}{\hyperfalse}
\hypertrue
\ifhyper\usepackage{hyperref}\fi

\usepackage{enumitem}

\usepackage{framed}
\usepackage{nicefrac}
\usepackage[noabbrev,capitalise,nameinlink]{cleveref}
\usepackage{cleveref}

\newtheorem{theorem}{Theorem}[section]

\newtheorem{cond}[theorem]{Condition}
\newtheorem{lemma}[theorem]{Lemma}
\newtheorem{informal theorem}[theorem]{Theorem (informal statement)}

\newtheorem{proposition}[theorem]{Proposition}

\newtheorem{claim}[theorem]{Claim}
\newtheorem{fact}[theorem]{Fact}

\theoremstyle{definition}
\newtheorem{definition}[theorem]{Definition}
\newcommand{\eqdef}{\stackrel{{\mathrm {\footnotesize def}}}{=}}

\newcommand{\bx}{\mathbf{x}}
\newcommand{\by}{\mathbf{y}}
\newcommand{\bv}{\mathbf{v}}
\newcommand{\bw}{\mathbf{w}}

\newcommand{\bX}{\mathbf{X}}
\newcommand{\bY}{\mathbf{Y}}

\newcommand{\adv}{\mathrm{Mass}}
\newcommand{\daniel}{\mathrm{Mass-L2}}

\newcommand{\loss}{\mathcal{L}_{\text{2}}}

\newcommand{\I}{\mathbb{I}}
\newcommand{\p}{\mathbf{P}}
\newcommand{\q}{\mathbf{Q}}

\newcommand{\R}{\mathbb{R}}

\newcommand{\Z}{\mathbb{Z}}

\newcommand{\E}{\mathbf{E}}

\newcommand{\dtv}{d_{\mathrm TV}}
\newcommand{\pr}{\mathbf{Pr}}
\renewcommand{\Pr}{\mathbf{Pr}}
\newcommand{\poly}{\mathrm{poly}}
\newcommand{\var}{\mathbf{Var}}

\newcommand{\relu}{\mathrm{ReLU}}

\newcommand{\sgn}{\mathrm{sign}}
\newcommand{\sign}{\mathrm{sign}}

\newcommand{\opt}{\mathrm{OPT}}
\newcommand{\D}{\mathcal{D}}

\newcommand{\bin}{\mathrm{Bin}}

\newcommand{\bi}{\mathbf{i}}

\newcommand{\wt}{\widetilde}
\newcommand{\wh}{\widehat}

\author{
Ilias Diakonikolas\thanks{Supported by NSF Medium Award CCF-2107079,
NSF Award CCF-1652862 (CAREER), a Sloan Research Fellowship, and
a DARPA Learning with Less Labels (LwLL) grant.}\\
University of Wisconsin-Madison\\
{\tt ilias@cs.wisc.edu}\\
\and
Daniel M. Kane\thanks{Supported by NSF Medium Award CCF-2107547,
NSF Award CCF-1553288 (CAREER), a Sloan Research Fellowship, and a grant from CasperLabs.}\\
University of California, San Diego\\
{\tt dakane@cs.ucsd.edu}\\
\and
Lisheng Ren \thanks{Supported by NSF Award CCF-1652862 (CAREER).}\\
University of Wisconsin-Madison\\
{\tt lren29@wisc.edu}\\
\and
Yuxin Sun \thanks{Supported by NSF Award CCF-1652862 (CAREER).}\\
University of Wisconsin-Madison\\
{\tt yxsun@cs.wisc.edu}\\
}

\title{SQ Lower Bounds for Learning Single Neurons \\ with Massart Noise}

\begin{document}

\maketitle

\begin{abstract}
We study the problem of PAC learning a single neuron in the presence of Massart noise. 
Specifically, for a known activation function $f: \R \to \R$, the learner is given
access to labeled examples $(\bx, y) \in \R^d \times \R$, where the marginal distribution of $\bx$ is 
arbitrary and the corresponding label $y$ is a Massart corruption of $f(\langle \bw, \bx \rangle)$.
The goal of the learner is to output a hypothesis $h: \R^d \to \R$ with small squared loss.
For a range of activation functions, including ReLUs, 
we establish super-polynomial Statistical Query (SQ) lower bounds for this learning problem. 
In more detail, we prove that no efficient SQ algorithm can approximate 
the optimal error within any constant factor. 
Our main technical contribution is a novel SQ-hard construction 
for learning $\{ \pm 1\}$-weight Massart halfspaces on the Boolean hypercube 
that is interesting on its own right.
\end{abstract}

\setcounter{page}{0}

\thispagestyle{empty}

\newpage

\section{Introduction} \label{sec:intro}

The success of deep learning has served as a motivation 
for understanding the complexity of learning simple classes of neural networks.
Here we study arguably the simplest possible setting of learning a {\em single} neuron,
i.e., a real-valued function of the form $\mathbf{x} \mapsto f(\langle \mathbf{w} , \mathbf{x} \rangle)$, 
where $\mathbf{w}$ is the weight vector of parameters and $f: \R \mapsto \R$ is a 
non-linear and monotone activation. The underlying learning problem is the following:
Given i.i.d.\ samples from a distribution
$\D$ on $(\mathbf{x}, y)$, where $\mathbf{x} \in \R^d$ is the example and $y \in \R$ 
is the corresponding label, the goal is to learn the target function in $L_2^2$-loss.
That is, the objective of the learner is to output a hypothesis
$h: \R^d \mapsto \R$ such that $\E_{(\mathbf{x}, y) \sim \D} [(h(\mathbf{x}) - y)^2]$ 
is as small as possible, compared to the optimal loss value 
$\opt: = \min_{\mathbf{w} \in \R^d} \E_{(\mathbf{x}, y) \sim \D} [(f( \langle \mathbf{w} , \mathbf{x} \rangle) - y)^2]$. 
A learning algorithm in this context is called proper if the hypothesis $h$
is restricted to be of the form 
$h_{\widehat{\mathbf{w}}}(\mathbf{x}) = f(\langle \widehat{\mathbf{w}},\mathbf{x}\rangle)$, for some 
$\widehat{\mathbf{w}} \in \R^d$.
One of the most popular activations is the ReLU function, corresponding to
$f(u) = \mathrm{ReLU}(u) \eqdef \max\{0, u\}$. 
In this work, we study the complexity of improperly learning single neurons, 
where the marginal distribution on examples
is fixed but arbitrary and the hypothesis $h$ is allowed to be any efficiently computable function.

In the realizable case, i.e., when the labels are consistent with a function 
in the target concept class,
the above learning problem is known to be efficiently solvable
for various activation functions. A line of work, see, e.g.,~\cite{KalaiS09, Mahdi17, yehudai2020learning}
and references therein, has shown that simple algorithms like gradient-descent 
efficiently converge to an optimal solution (in some cases under assumptions 
on the marginal distribution on examples). 
On the other hand, in the adversarial label noise (aka agnostic) model, 
known hardness results~\cite{Daniely16, DKMR22} rule out efficient constant factor approximations 
to the optimal loss for a range of activations including ReLUs.
The aforementioned negative results for label agnostic learning
motivate the study of weaker corruption models, where non-trivial efficient learning 
algorithms may still be possible. A natural class of such models --- that may be more realistic 
in some practical applications --- are semi-random noise models, 
involving a combination of adversarial choices and random choices.

Here we focus on the {\em Massart (or bounded) noise} model~\cite{Massart2006}, 
a classical semi-random model first defined in the context of binary classification 
(see~\cite{Sloan88} for an equivalent noise model).
Intuitively, in the Massart model, an adversary has control over a (uniformly) {\em random} $\eta<1/2$ fraction of the labels 
(see~\cref{def:Massart-L0}). 
In the context of binary classification,~\cite{DGT19} gave the first non-trivial 
learning algorithm for halfspaces in this model (see also~\cite{CKMY20, DKT21}). 
Subsequent work~\cite{DK21-SQ-Massart, NT22} provided evidence 
that the error guarantee of the latter algorithm is essentially best possible 
in the Statistical Query (SQ) model~\cite{Kearns:98}; and, 
more recently, under standard cryptographic assumptions~\cite{DKMR22-massart}.

To state our results, we formally define the following natural generalization of the Massart model
for real-valued functions (see, e.g.,~\cite{ChenKMY21, DPT21}).

\begin{definition}[Massart Noise Model] \label{def:Massart-L0}
Let $\mathcal{G}$ be a concept class of real-valued functions over $\R^d$,  
$\mathcal{D}_\bx$ be a fixed distribution over $\R^d$, and $0<\eta<1/2$. 
Fix an unknown function $g \in \mathcal{G}$.
The noiseless distribution $\D$ (corresponding to $g$) 
is the distribution on labeled examples $(\bX,Y)$, 
supported on $\R^d \times \R$, where $\bX \sim\D_\bx$ and 
$Y=g(\bX)$. An \emph{$\eta$-Massart distribution},  $\D^{\adv}_\eta$, is 
a distribution on labeled examples $(\bX,Y')$, supported on $\R^d \times \R$,
such that for $(\bX,Y') \sim \D^{\adv}_\eta$ we have that
(i) $\bX\sim\D_\bx$, and (ii) for all $\bx\in\R^d$ it holds that
$\Pr_{(\bX,Y')\sim \D^{\adv}_\eta}[Y'\ne g(\bX) \mid \bX=\bx]\le\eta$.
\end{definition}

Given sample access to the $\eta$-Massart distribution $\D^{\adv}_\eta$, 
corresponding to an unknown $g \in \mathcal{G}$,
the goal  is to output a hypothesis $h:\R^d \mapsto \R$ such that 
$\loss(h; \D^{\adv}_\eta) : = \E_{(\bX,Y') \sim \D^{\adv}_\eta}[(Y'-h(\bX))^2]$ is small. 
Let $\opt_\adv:=\inf_{g\in\mathcal{G}}\E_{(\bX,Y')\sim \D^{\adv}_\eta}[(Y'-g(\bX))^2]$ denote 
the optimal squared error. We will say that a learning algorithm is {\em $\alpha$-approximate} 
if it outputs a hypothesis $h: \R^d \mapsto \R$ that with high probability 
satisfies $\loss(h; \D^{\adv}_\eta) \leq \alpha(d)\cdot \opt_\adv$. We say that a learner
is a {\em constant factor approximation} if $\alpha = O(1)$.
Here we focus on the concept class of {\em single neurons}: 
for an activation $f: \R \mapsto \R$, we will denote by 
$\mathcal{C}_f \eqdef \{ c_{\bw}: \R^d \mapsto \R \mid c_{\bw}(\mathbf{x})=f(\langle \mathbf{w},\mathbf{x}\rangle), \mathbf{w}\in \R^d  \}$.

In~\cref{def:Massart-L0}, the Massart adversary corrupts each label independently 
with probability {\em at most} $\eta$, for some $\eta<1/2$.
Even though this noise model might appear innocuous, 
the fact that the corruption probability is unknown to the learner 
makes the design of efficient Massart learners challenging. 
The Massart model has been extensively studied
in the context of binary classification~\cite{Sloan88, RivestSloan:94, 
Sloan96, AwasthiBHU15, AwasthiBHZ16, DGT19, DKTZ20, CKMY20, DKT21} 
and, more recently,
for learning real-valued functions~\cite{ChenKMY21, DPT21}.

For the task of PAC learning halfspaces with Massart noise (i.e., neurons
corresponding to the sign activation), 
there is compelling evidence that even approximate learning 
is computationally hard~\cite{DK21-SQ-Massart, NT22, DKMR22-massart}. 
In sharp contrast, our
understanding of the possibilities and limitations 
of Massart learning well-behaved {\em real-valued} functions
(including ReLUs and other Lipschitz monotone activations) remains limited. 
On the positive side, recent work developed the first efficient learners for linear regression~\cite{ChenKMY21, DPT21} 
and ReLU regression~\cite{DPT21} with Massart noise. We note that 
the ReLU regression algorithm of~\cite{DPT21} requires a certain anti-concentration condition 
on the distribution $\D_{\bx}$ of examples, which is crucial for its performance guarantees. 
In fact, without such an assumption, no non-trivial upper bound 
is known for ReLUs (or other non-linear activations).
This discussion
prompts the following question:
\begin{center}
{\em Is there an efficient $O(1)$-approximate learner for {\em distribution-free} learning\\ 
of a single neuron with Massart noise?}
\end{center}
For the important case of ReLU activations, \cite{DPT21} conjectured that
the distribution-independent PAC learning problem is intractable. As the main contribution of this paper,
we provide strong evidence towards this conjecture, by establishing super-polynomial lower bounds 
in the Statistical Query (SQ) model --- a restricted but powerful family of algorithms. Specifically, 
we show that no efficient SQ algorithm can achieve {\em any} constant factor approximation.
Moreover, our SQ-hardness result is not specific to ReLUs, but generalizes
to a broad class of non-linear activation functions.

\subsection{Our Results} \label{ssec:results}

In this work, we give strong evidence that the problem of learning single neurons with Massart
noise does not admit any constant factor approximation. Specifically, we show that any 
efficient SQ algorithm cannot achieve a constant factor approximation. In fact, the hardness
gap that we establish is super-constant, scaling with the dimensionality of the problem.

Instead of directly accessing samples, SQ algorithms~\cite{Kearns:98} are only to adaptively 
query expectations of bounded functions of the underlying distribution 
up to some tolerance (see~\cref{sec:prelims}). 
The class of SQ algorithms is fairly broad: a wide range of known algorithmic techniques in
machine learning are known to be implementable in the SQ model~\cite{FeldmanGRVX17}.

For the important class of ReLU activations,  
our main result is the following:

\begin{theorem}[SQ Hardness of Massart Learning ReLUs]\label{thm:sq-relu-informal}
Any SQ algorithm that learns a single neuron with ReLU activation on $\R^d$,  in the presence of Massart noise 
with $\eta=1/3$, to squared error better than $1/\poly(\log(d))$ requires either queries of 
accuracy better than $2^{-(\log d)^{c_1}}$ or at least $2^{(\log d)^{c_1}}$ statistical queries, for some constant $c_1>1$. 
This holds even if the optimal squared error is at most $2^{-(\log d)^{c_2}}$ for some $0<c_2<1$, 
and the total weight of the neuron is $\poly(d)$.
\end{theorem}

\cref{thm:sq-relu-informal} rules out the existence of efficient SQ algorithms 
(i.e., using polynomially many queries of inverse polynomial accuracy) 
with approximation ratio $2^{(\log d)^c}$ for some $0<c<1$. It therefore a fortiori rules
out any constant factor approximate SQ learner.

We note that the SQ-hardness result of~\cref{thm:sq-relu-informal} does not require
the linearity of the ReLU (on positive inputs); a similar result can be shown 
for a broader class of activation functions. Specifically, we can generalize our SQ-hardness result to any
activation $f$ of the form $f(u) = 0$, $u<0$, and $\exists u_0\ge0,f(u_0)\ne0$.

\cref{thm:sq-relu-informal}  establishes SQ-hardness of learning single neurons
under the Massart noise notion of~\cref{def:Massart-L0}. 
We note that for learning real-valued functions, one can consider
other natural definitions of ``Massart noise''. Specifically,~\cref{def:Massart-L0} considers an 
$L_0$-perturbation (the adversary is allowed to arbitrarily corrupt a random $\eta$-fraction of the labels).
Another natural definition considers $L_2$-perturbations, as stated below (note that in the definition 
below, the parameter $\eta$ does not need to be bounded above by $1/2$).

\begin{definition}[$L_2$-Massart Noise Model] \label{def:Massart-L2}
Let $\mathcal{G}$ be a concept class of real-valued functions over $\R^d$,  
$\mathcal{D}_\bx$ be a fixed distribution over $\R^d$, and $\eta >0$. 
Fix an unknown function $g \in \mathcal{G}$.
An \emph{$\eta$-$L_2$-Massart distribution},  $\D^\daniel_\eta$, is 
a distribution on labeled examples $(\bX,Y)$, supported on $\R^d \times \R$,
such that for $(\bX,Y) \sim \D^\daniel_\eta$ we have that
(i) $\bX\sim\D_\bx$, and (ii) for all $\bx\in\R^d$ it holds that
$\E_{(\bX,Y) \sim \D^\daniel_\eta}[(Y-g(\bX))^2 \mid \bX=\bx]\le 4\eta.$
We will use $\opt_\daniel:=\inf_{g\in\mathcal{G}}\E_{(\bX,Y)\sim \D^\daniel_\eta}[(Y-g(\bX))^2]$ 
to denote the optimal squared error.
\end{definition}

Note that for $\{\pm 1\}$ labels, with noise rate $\eta <1/2$, 
the above model generalizes the standard Massart model (for binary classification) 
with the same noise rate $\eta$. For this noise model, we establish SQ-hardness
for the following general family of non-linear activations (including ReLUs):

\begin{definition}[Fast Convergent Activation] \label{def:fast-conv-act}
We say that a function $f:\R \mapsto \R$ is a fast-convergent activation 
if either $g(t):=f(t)$ or $g(t):=f(-t)$ satisfies the following: 
(i) $\lim_{t\rightarrow -\infty}g(t)$ exists. (ii) For $t<0$ with absolute value sufficiently large, 
$|g(t)-g(-\infty)|=1/\poly(|t|)$.
\end{definition}

Intuitively, the second condition above requires that the function converges to its limit at inverse polynomial rate.
Without loss of generality, we consider activations which converge on the negative side.
For such an activation $f$, let $f_-:=f(-\infty)$ and $c_+$ be a constant such that $f(c_+)\ne f_-$.

Our proof technique establishing~\cref{thm:sq-relu-informal} is quite robust  
and can be adapted to $L_2$-Massart noise under fast convergent activations.
Our main result in this context is the following:

\begin{theorem}[SQ Hardness of $L_2$-Massart Learning] \label{thm:sq-fc-L2}
Let $f: \R \mapsto \R$ be a fast convergent activation.
Any SQ algorithm that learns a single neuron with activation $f$ on $\R^d$,  
in the presence of $\eta$-$L_2$-Massart noise 
with $\eta=\frac{2(f(c_+)-f_-)^2}{9}$, to squared error 
better than $1/\poly(\log(d))$ requires either queries of 
accuracy better than $2^{-(\log d)^{c_1}}$ or 
at least $2^{(\log d)^{c_1}}$ statistical queries, for some constant $c_1>1$. 
This holds even if the optimal squared error is 
at most $2^{-(\log d)^{c_2}}$ for some $0<c_2<1$,
and the total weight of the neuron is $\poly(d)$.
\end{theorem}

\noindent Interestingly, the key ingredient for our aforementioned 
SQ hardness results for real-valued functions is a new SQ hardness 
construction for low-weight halfspaces 
(i.e., neurons with a $\sgn$ activation) on the Boolean hypercube. 
In this context, we prove:

\begin{theorem}[SQ Hardness for Low-weight Massart Halfspaces on $\{0, 1\}^d$] \label{thm:discrete-ltfs-intro}
Any SQ algorithm that learns $\{\pm 1\}$-weight halfspaces on $\{0,1\}^d$,  
in the presence of Massart noise with $\eta=1/3$, 
to 0-1 error better than $1/\poly(\log(d))$ requires either queries of 
accuracy better than $2^{-(\log d)^{c_1}}$ or 
at least $2^{(\log d)^{c_1}}$ statistical queries, for some constant $c_1>1$. 
This holds even if the optimal 0-1 error 
is at most $2^{-(\log d)^{c_2}}$ for some $0<c_2<1$.
\end{theorem}

\cref{thm:discrete-ltfs-intro} rules out any efficient polynomial (relative) approximation
for  $\{\pm 1\}$-weight halfspaces on the hypercube. This
is the first hardness result for approximate learning of Boolean Massart halfspaces. 
Prior work either obtained SQ-hardness of exact learning~\cite{CKMY20} 
or was inherently applicable to halfspaces on $\R^d$~\cite{DK21-SQ-Massart, NT22}. 

A number of learning problems involving halfspaces  
are computationally easy when the weights are small integers (aka in the ``large margin'' case) 
and computationally hard for arbitrary weights. 
The conceptual message of~\cref{thm:discrete-ltfs-intro} 
is that the Massart halfspace learning problem is hard 
due to the combinatorial nature of the problem 
(and not due to the magnitude of the weights).
This addresses an open problem of~\cite{Blum03} 
regarding the complexity of Massart learning 
simple halfspaces.

\subsection{Technical Overview} \label{ssec:techniques}

We start by describing the proof of~\cref{thm:discrete-ltfs-intro}.
Our SQ lower bound for learning Boolean Massart halfspaces requires a number of novel ideas. 
Our starting point is the construction of~\cite{DK21-SQ-Massart} 
that proves a similar lower bound in the continuous setting. 
They begin by producing a one-dimensional construction 
of a Massart polynomial threshold function (PTF) 
whose distributions conditional on $y=1$ and $y=-1$ 
approximately match many moments. 
Using techniques from~\cite{DKS17-sq}, 
they show that by embedding this one-dimensional construction 
into higher dimensions, 
they can produce $d$-dimensional instances of Massart PTFs 
that are SQ-hard to learn. 
They then further embed these instances via the Veronese embedding 
to produce SQ-hard LTF instances 
(essentially using the fact that a PTF in $x$ 
is an LTF in the low-degree monomials of $x$).

Our proof adapts this general idea to the discrete setting. 
The first obstacle is developing an appropriate analogue of the one-dimensional construction. 
The construction from~\cite{DK21-SQ-Massart} uses the fact that a discrete Gaussian 
nearly matches moments with a standard Gaussian; 
thus, making the conditional distributions 
of $x$ mixtures of discrete Gaussians 
ensures that the moment-matching condition is satisfied. 
By carefully picking this mixture, 
they ensure that the conditional distributions have no overlap for $|x|$ small 
(thus ensuring a small value of $\opt$), 
but that the $y=1$ case is always more likely for $|x|$ sufficiently large. 
This construction does not work in our setting, 
as we need our one-dimensional instance to be discrete.

Our basic idea is to begin by noting that the binomial distribution 
conditioned on $x$ being $0 \mod s$ 
approximately matches many moments with the full binomial. 
As a first attempt, we let $y=-1$ if $x\equiv0 \mod s$ and $y=1$ otherwise. 
This matches many moments with the binomial, 
but alternates between $y=1$ and $y=-1$ many times, 
and thus cannot be considered to be a low-degree 
PTF with Massart noise.
To fix this issue, we need to modify our distributions so that:
(i) Conditioned on any $x$ far from $n/2$, $y$ is more likely to be $1$ than $-1$,
(ii) the two distributions conditioning on $y=1$ and $y=-1$ have little overlap, and
(iii) each conditional distribution approximately matches moments with the full binomial.
We can fix (i) at the cost of (ii) by replacing the conditional distribution on $y=1$ 
with the full binomial distribution.
As long as the prior probability of $y=1$ exceeds that of $y=-1$ by enough,
even for $x\equiv0 \mod s$, $y=1$ will be more likely than $y=-1$.
Unfortunately, the conditional distributions now have too much overlap. 
We can address this by moving the mass in the $y=1$ 
conditional off of the points 
with $x \equiv0 \mod s$ and $|x-n/2|$ small.
Importantly, we must find a way to do this 
without destroying property (iii). 
To that end, we show that there is a way to move mass 
from each of these points $x$ 
and redistribute it to nearby points in such a way so as to not affect 
any of the low-order moments (see~\cref{lem:duality-disc}).
By doing this to each $x \equiv 0 \mod s$ with $|x-n/2|$ small, 
we get our final construction.

We also need to modify the method by which 
we embed the one-dimensional construction 
into higher dimensions in order to obtain the family of SQ-hard PTF instances. 
This construction must differ from previous constructions, 
as our family of distributions will be discrete 
and not Gaussian-like as in~\cite{DKS17-sq}. 
Fortunately, we can leverage the 
recent technique of~\cite{DKS22}, embedding our low-dimensional 
construction as a junta.
In particular, a significant difference with the Gaussian case 
is in the way we embed the low-dimensional distribution as a higher-dimensional one.
In the Gaussian case, we simply take the distribution to be Gaussian in independent directions.
In our discrete setting, we begin by embedding into a moderate dimensional hypercube 
by taking the unique symmetric distribution,
where our one-dimensional distribution over some subset $S$ 
is the distribution over $\sum_{i\in S}X_i$.
We note that this distribution will approximately match low-degree moments 
with the uniform distribution over the hypercube.
We then embed this distribution into a higher-dimensional 
hypercube as a random junta.

As an application of the above general recipe to obtain 
SQ lower bounds for discrete distributions, 
we note that the hard instances we construct 
for learning Boolean halfspaces with Massart noise,
\emph{also} (with a slight change of variables) 
produce hard instances for ReLUs (and other activations). 
In particular, in our hard instance for PTFs, the optimal classifier $f$ is given by $f(\bx) = -1$ 
if $\bx_S \equiv 0$ $\mod s$ and $|\bx_S-n/2| < ds/2$, and $1$ otherwise, where $\bx_S$ 
is the sum over the coordinates of $\bx$ in some particular subset $S$. 
We note that the function $(1-f(\bx))/2$, which is equal to $1$ if $\bx_S \equiv 0 \mod s$ 
and $|\bx_S-n/2| < ds/2$ and $0$ otherwise, can be written as $\relu(p(\bx))$ 
for some degree $O(d)$ polynomial $p$, 
where $p(\bx) = 1$ for $\bx_S \equiv 0 \mod s$ 
and $|\bx_S-n/2| < ds/2$, and $p(\bx)\le0$ otherwise. 
By replacing $\bx$ by its Vernonese embedding as before, 
we can produce hard instances of ReLU functions with Massart noise.

\section{Preliminaries} \label{sec:prelims}

\paragraph{Notation} For $n \in \Z_+$, we denote $[n] \eqdef \{1,\ldots,n\}$.
For two distributions $p,q$ over a probability space $\Omega$,
let $\dtv(p,q)=\sup_{S\subseteq\Omega}|p(S)-q(S)|$
denote the total variation distance between $p$ and $q$.
We use $\Pr[\mathcal{E}]$ and $\mathbb{I}[\mathcal{E}]$ for
the probability and the indicator of event $\mathcal{E}$.
For a real random variable $X$, we use $\E[X],\var[X]$ to denote the expectation
and variance of $X$, respectively.
For $n\in\Z_+$ and $0\le p\le1$, we use $\bin(n,p)$
to denote the Binomial distribution with parameters $n$ and $p$.
Throughout this article, we will use capital letters (e.g., $X,\bX$) to denote random variables and random vectors, and small letters (e.g, $x,\bx$) to denote corresponding values.

\paragraph{Statistical Query Algorithms}
We will use the framework of Statistical Query (SQ) algorithms for problems 
over distributions~\cite{FeldmanGRVX17}.
We require the following standard definition.

\begin{definition}[Decision/Testing Problem over Distributions]\label{def:decision}
Let $D$ be a distribution and $\D$ be a family of distributions over $\R^M$. 
We denote by $\mathcal{B}(\mathcal{D},D)$ the decision (or hypothesis testing) 
problem in which the input distribution $D'$ is promised to satisfy either 
(a) $D'=D$ or (b) $D'\in\mathcal{D}$, and the goal of the algorithm 
is to distinguish between these two cases.
\end{definition}

\noindent We define SQ algorithms as algorithms that do not have direct access to samples from the distribution,
but instead have access to an SQ oracle. We will consider the following standard oracle.
\begin{definition}[$\mathrm{STAT}$ Oracle]\label{def:stat}
Let $D$ be a distribution on $\R^M$. A \emph{Statistical Query (SQ)} 
is a bounded function $f:\R^M\to[-1,1]$. For $\tau>0$, the $\mathrm{STAT}(\tau)$ 
oracle responds to the query $f$ with a value $v$ such that $|v-\E_{X\sim D}[f(X)]|\le\tau$. 
We call $\tau$ the \emph{tolerance} of the statistical query.
A \emph{Statistical Query (SQ) algorithm} is an algorithm 
whose objective is to learn some information about an unknown 
distribution $D$ by making adaptive calls to the corresponding $\mathrm{STAT}(\tau)$ oracle.
\end{definition}

\noindent To define the SQ dimension, we need the following definition.

\begin{definition}[Pairwise Correlation] \label{def:pc}
The pairwise correlation of two distributions with probability mass functions (pmfs)
$D_1, D_2 : \{0,1\}^M \to \R_+$ with respect to a distribution with pmf $D:\{0,1\}^M \to \R_+$,
where the support of $D$ contains the supports of $D_1$ and $D_2$,
is defined as $\chi_{D}(D_1, D_2) + 1 \eqdef \sum_{x\in\{0,1\}^M} D_1(x) D_2(x)/D(x)$.
We say that a collection of $s$ distributions $\mathcal{D} = \{D_1, \ldots , D_s \}$ over $\{0,1\}^M$
is $(\gamma, \beta)$-correlated relative to a distribution $D$ if
$|\chi_D(D_i, D_j)| \leq \gamma$ for all $i \neq j$, and $|\chi_D(D_i, D_j)| \leq \beta$ for $i=j$.
\end{definition}

\noindent The following notion of dimension effectively characterizes the difficulty of the decision problem.
\begin{definition}[SQ Dimension] \label{def:sq-dim}
For $\gamma ,\beta> 0$, a decision problem $\mathcal{B}(\mathcal{D},D)$, 
where $D$ is fixed and $\mathcal{D}$ is a family of distributions over $\{0,1\}^M$,
let $s$ be the maximum integer such that there exists
$\mathcal{D}_D \subseteq \D$ such that $\D_D$ is $(\gamma,\beta)$-correlated
relative to $D$ and $|\D_D|\ge s$.
We define the {\em Statistical Query dimension} with pairwise correlations $(\gamma, \beta)$
of $\mathcal{B}$ to be $s$ and denote it by $\mathrm{SD}(\mathcal{B},\gamma,\beta)$.
\end{definition}

\noindent The connection between SQ dimension and lower bounds is captured by the following lemma.

\begin{lemma}[\cite{FeldmanGRVX17}] \label{lem:sq-from-pairwise}
Let $\mathcal{B}(\D,D)$ be a decision problem, 
where $D$ is the reference distribution and $\D$ is a class of distributions over $\{0,1\}^M$. 
For $\gamma, \beta >0$, let $s= \mathrm{SD}(\mathcal{B}, \gamma, \beta)$.
Any SQ algorithm that solves $\mathcal{B}$ with probability at least $2/3$ 
requires at least $s \cdot \gamma /\beta$ queries to the
$\mathrm{STAT}(\sqrt{2\gamma})$ oracles.
\end{lemma}

\section{SQ Hardness Construction for Supervised Learning}\label{sec:SQ-super}

\subsection{Generic SQ Lower Bound Construction} \label{ssec:generic-discrete}
We start with some basic definitions. Let $U_M$ be the uniform distribution over $\{0,1\}^M$.
For a subset $T\subseteq[M]$ and $\bx \in \{0, 1\}^M$, we denote $\chi_T(\bx) = (-1)^{\sum_{i \in T}x_i}$.
For a distribution $\p$ over $\{0,1\}^M$, let $\wh{\p}(T) = \E_{\bX\sim\p}[\chi_T(\bX)]$.
We will require the orthogonal polynomials under the binomial distribution.
\noindent We have the following fact about the chi-squared inner product in the discrete setting.
\begin{fact} \label{fact:chi}
For distributions $\p, \q$ over $\{0,1\}^M$, we have that
$1+\chi_{U_M}(\p, \q) = \sum_{T\subseteq [M]} \wh{\p}(T) \wh{\q}(T).$
\end{fact} 

\begin{definition}[Kravchuk Polynomial~\cite{Szego:39}] \label{def:Krav}
For $k, m, x \in \Z_+$ with $0\le k,x\le m$, 
the Kravchuk polynomial $\mathcal{K}_k(x;m)$ 
is the univariate degree-$k$ polynomial
in $x$ defined by 
$\mathcal{K}_k(x;m) := \sum_{T\subseteq[m], |T|=k} \chi_T(\by)=\sum_{j=0}^k(-1)^j\binom{x}{j}\binom{m-x}{k-j}$,
where $\by$ has $x$ 1's and $m-x$ 0's.
\end{definition}

\noindent The following distribution family that is the basis of 
our discrete SQ lower bound construction.

\begin{definition} [High-Dimensional Hidden Junta Distribution] \label{def:p-hidden-junta}
Let $m, M \in \Z_+$ with $m<M$.
For a distribution $A$ on $[m]\cup\{0\}$ with probability mass function (pmf) $A(x)$ 
and a subset $S\subseteq [M]$ with $|S|=m$, 
consider the probability distribution over $\{0,1\}^M$, 
denoted by $\p^A_S$, such that for $\bX \sim \p^A_S$ the distribution 
$(X_i)_{i \not\in {S}}$ is the uniform distribution on its support and the distribution
$(X_i)_{i \in S}$ is symmetric with $\sum_{i\in S} X_i$ distributed according to $A$.
Specifically, $\p^A_S$ is given by the pmf
\begin{align*}
\p^A_S(\bx) = 2^{-M+m} A\left(\sum_{i\in S}x_i \right)\binom{m}{\sum_{i\in S}x_i}^{-1} \;.
\end{align*}
\end{definition}

\noindent We now define the hypothesis testing and learning problem
which will be used throughout this paper:
\begin{definition}[Hidden Junta Binary Testing Problem]\label{def:testing}
Fix $a\ne b\in\R$. Let $A$ and $B$ be distributions on $[m]\cup\{0\}$ satisfying~\cref{cond:moments-disc} with
parameters $k \in \Z_+$ and $\nu \in \R_+$, and let $p\in (0,1)$.
For $M \in \Z_+$, $M>m$, and a subset $S\subseteq [M]$ with $|S|=m$,
define the distribution $\p^{A,B,p}_{S,a,b}$ on $\{0,1\}^M\times \{a,b\}$ that returns a sample
from $(\p^A_S,a)$ with probability $p$ and a sample from $(\p^B_S,b)$ with probability $1-p$.
In the $(A,B,a,b,M)$-Hidden Junta Testing Problem,
one is given access to a distribution $D$ so that either $H_0$: $D=U_M^p$,
where for $(\bX,Y)\sim U_M^p$ we have that $\bX$ is a uniform random element of $\{0,1\}^M$,
and $Y$ is independently $a$ with probability $p$ and $b$ with probability $1-p$.
$H_1$: $D$ is given by $\p_{S,a,b}^{A,B,p}$ for some subset $S\subseteq[M]$ with $|S|=m$.
One is then asked to distinguish between $H_0$ and $H_1$.
\end{definition}

Note that this is just the hypothesis testing problem $\mathcal{B}(\D,D)$ 
with $D=U_M^p$ and $\D=\{\p_{S,a,b}^{A,B,p}\}$.
The following condition describes the approximate moment-matching property
of the desired distribution $A$ with the Binomial distribution.

\begin{cond} \label{cond:moments-disc}
Let $k,m \in \Z_+$ with $k < m$ and $\nu>0$.
The distribution $A$ on $[m]\cup\{0\}$ is such that 
$|\E_{X\sim A}[\mathcal{K}_t(X;m)]|\leq \nu$, for all $1\leq t \leq k$.
\end{cond}

The following correlation lemma states that the distributions $\p_S^A$ 
are nearly orthogonal as long as $A$ satisfies the nearly moment-matching condition.

\begin{lemma}[Correlation Lemma~\cite{DKS22}] \label{lem:cor-disc}
Let $k, m, M\in \Z_+$ with $k \leq m \leq M$.
If the distribution $A$ on $[m]\cup\{0\}$ satisfies~\cref{cond:moments-disc},
then for all $S,S' \subseteq [M]$ with $|S| = |S'| = m$, we have that
\begin{equation*} 
|\chi_{U_M}(\p^A_S, \p^A_{S'})| \le (|S\cap S'|/m)^{k+1} \chi^2(A, \mathrm{Bin}(m,1/2)) + k\nu^2 \;.
\end{equation*}
\end{lemma}

We will also use the following standard fact:

\begin{fact}\label{fact:near-orth-vec-disc}
Let $m, M \in \Z_+$ with $m<M$.
For any constant $0< c< 1$ and $M> 2m/c$,
there exists a collection $\cal{C}$ of $2^{\Omega_c(m)}$ subsets $S\subseteq [M]$
such that any pair $S,S'\in\cal{C}$, with $S \neq S'$, satisfies $|S\cap S'|<cm$.
\end{fact}

In fact, an {appropriate size} set of random subsets
satisfies the above statement {with high probability}.

\begin{proposition}[Generic Discrete SQ Lower Bound]\label{prop:gen-sq-prop}
Let $m,M\in\Z_+$ with $M>m$. Let $A,B$ be distributions on $[m]\cup\{0\}$ satisfying~\cref{cond:moments-disc}. 
Let $\tau\ge k \nu^2 + 2^{-k}(\chi^2(A, \mathrm{Bin}(m,1/2))+\chi^2(B, \mathrm{Bin}(m,1/2)))$.
Any SQ algorithm that solves the testing problem of~\cref{def:testing} with probability at least 2/3 must
either make queries of accuracy better than $\sqrt{2\tau}$ or must make at least
$2^{\Omega(m)}\tau/(\chi^2(A, \mathrm{Bin}(m,1/2))+\chi^2(B, \mathrm{Bin}(m,1/2)))$ statistical queries.
\end{proposition}
\begin{proof}
Let $\mathcal{C}$ be a collection of $s=2^{\Omega(m)}$ subsets $S\subseteq[M]$
with $|S|=m$ whose pairwise intersections are all less than $m/2$.
By~\cref{fact:near-orth-vec-disc} (taking the local parameter $c=1/2$), 
such a set is guaranteed to exist. We then need to show that for $S , S' \in \cal{C}$, we have
that $|\chi_{U^p_M}(\p^{A,B,p}_{S,a,b},\p^{A,B,p}_{S',a,b})|$ is small.
Since $U^p_M, \p^{A,B,p}_{S,a,b},$ and $\p^{A,B,p}_{S',a,b}$ all assign $y=1$ with probability $p$,
it is not hard to see that
\begin{align*}
\chi_{U^p_M}(\p^{A,B,p}_{S,a,b},\p^{A,B,p}_{S',a,b})
= & \; p \; \chi_{U^p_M \mid y=1}\left( (\p^{A,B,p}_{S,a,b} \mid y=1) , (\p^{A,B,p}_{S',a,b} \mid y=1) \right) + \\
& (1-p) \; \chi_{U^p_M \mid y=-1} \left( (\p^{A,B,p}_{S,a,b} \mid y=-1) , (\p^{A,B,p}_{S',a,b} \mid y=-1) \right)\\
 = & \; p \; \chi_{U_M}(\p^A_S , \p^A_{S'}) + (1-p) \; \chi_{U_M}(\p^B_S, \p^B_{S'}).
\end{align*}
By~\cref{lem:cor-disc}, for $S,S'\in\mathcal{C}$ with $S\ne S'$, it holds that
$$\chi_{U^p_M}(\p^{A,B,p}_{S,a,b},\p^{A,B,p}_{S',a,b})
\leq k\nu^2 + 2^{-k}(\chi^2(A, \mathrm{Bin}(m,1/2))+\chi^2(B, \mathrm{Bin}(m,1/2)))\le\tau \;.$$
If $S=S'$, a similar computation shows that
$$\chi_{U^p_M}(\p^{A,B,p}_{S,a,b},\p^{A,B,p}_{S,a,b}) = \chi^2(\p^{A,B,p}_{S,a,b},U_M^p)
\leq \chi^2(A, \mathrm{Bin}(m,1/2))+\chi^2(B, \mathrm{Bin}(m,1/2)) \;.$$
Let $\gamma = \tau$ and
$\beta = \chi^2(A,\mathrm{Bin}(m,1/2))+\chi^2(B, \mathrm{Bin}(m,1/2))$. 
We have that the Statistical Query dimension of this testing problem
with correlations $\left(\gamma,\beta\right)$ is at least $s$.
Then applying~\cref{lem:sq-from-pairwise}
with $(\gamma,\beta)$ completes the proof.
\end{proof}

\subsection{Construction of Univariate Moment-Matching Distributions} \label{ssec:one-dim-disc}
Here we give the construction of our approximate moment-matching distributions. 
For convenience, we use the ``expectation'' and ``moments'' 
for the unnormalized measure without clarification.
The main result of this section is captured in the following proposition.

\begin{proposition}\label{prop:mainProp-disc}
Let $d, k, s, m \in \Z_+$ and $\zeta \in {(0, 1/2)}$ such that: 
(i) $s\ge\omega(k^4)$, (ii) $k < m/2$, 
(iii) $ds \ge\Omega( \sqrt{m\log(1/\zeta)})$, and (iv) $s^2d \le o(m)$.
There exist measures $\D_{+}$ and $\D_{-}$ over $[m]\cup\{0\}$ 
and a union $J$ of $d$ points in $[m]\cup\{0\}$ such that:
\begin{enumerate}
\item \label{prop:1-disc} (a) $\D_{+}=0$ on $J$, and (b) $\D_{+} > 2\D_{-}$ on $\overline{J} = {[m]\cup\{0\} \setminus J}$.
\item \label{prop:2-disc} All but $\zeta$-fraction of the measure of $\D_{-}$ lies in $J$. 
\item \label{prop:3-disc} The distributions $\D_{+}/\|\D_{+}\|_1$ and $\D_{-}/\|\D_{-}\|_1$ 
satisfy~\cref{cond:moments-disc} with parameters $k$ and $\nu \leq \binom{m}{k}\exp(-\Omega(m/s^2))$.
\item \label{prop:4-disc} (a) $\D_{+}$ is at most $O(1)\mathrm{Bin}(m,1/2)$ and (b) $\|\D_{+} \|_1 = \Theta(1)$.
\item \label{prop:5-disc} $\|D_{-}\|_1 = \Theta(1/s)$.
\end{enumerate}
\end{proposition}

\begin{proof}
We start by constructing each measure in turn.

\noindent \textbf{Definition of the Measure $\D_{-}$.}
We define the measure $\D_{-}$ as follows: $\D_-(x):=\bin(m,1/2)(x)$ 
if $x\equiv 0 \pmod{s}$; otherwise $\D_-(x)=0$.
We claim that this satisfies~\cref{cond:moments-disc}.
This is shown in the following lemma.
\begin{lemma}\label{lem:D-}
$\D_{-}(x)$ satisfies~\cref{cond:moments-disc} with parameters
$k$ and $\nu = s\binom{m}{k}\exp(-\Omega(m/s^2))$.
\end{lemma}

\begin{proof}
We need to bound $\E_{Z \sim \D_-} [\mathcal{K}_t(Z;m)]$ for $1\leq t\leq k$. By definition, we have that
\begin{align*}
\E_{Z\sim\D_-}[\mathcal{K}_t(Z;m)]=\sum_{T\subseteq[m],|T|=t}\E_{Z\sim\D_-}[\chi_T(\bY)] 
=\binom{m}{t}\E_{\bX\sim\mathcal{R}}[\chi_{T_0}(\bX)] \;,
\end{align*}
where $\bY$ has $Z$ 1's and $m-Z$ 0's, and
$\mathcal{R}\in \{0,1\}^m$ is the unique symmetric measure with $\sum_{i=1}^m X_i$
having measure $\D_-$, and $T_0 \subseteq [m]$ is some subset with $|T_0|=t$.
Let $\omega$ be a primitive $s^{th}$ root of unity. 
We note that the pmf $\mathcal{R}(\bx)$ of the measure $\mathcal{R}$
satisfies
\begin{equation*}
\mathcal{R}(\bx) = \frac{1}{2^m s} \sum_{j=0}^{s-1} \omega^{\left( j \sum_{i=1}^m x_i\right)}
= \frac{1}{2^m s} \sum_{j=0}^{s-1} \prod_{i=1}^m \omega^{j x_i} \;.
\end{equation*}
Therefore, we can write
\begin{equation*}
\mathcal{R}(\bx)\chi_{T_0}(\bx) =\left((-1)^{\sum_{i=1}^{m}\I[i\in T_0]x_i}\right)\Big(\frac{1}{2^m s}\sum_{j=0}^{s-1} 
\prod_{i=1}^m \omega^{j x_i}\Big)=\frac{1}{2^m s}\sum_{j=0}^{s-1} \prod_{i=1}^m \left(\omega^j (-1)^{\I[i\in T_0]} \right)^{x_i} \;.
\end{equation*}
Since the expectation is the sum of the above over all $x\in \{0,1\}^m$ and since this separates as a product,
we get that
\begin{equation*}
\E_{\bX\sim\mathcal{R}}[\chi_{T_0}(\bX)] = \frac{1}{2^m s}\sum_{\bx\in\{0,1\}^m}\sum_{j=0}^{s-1} 
\prod_{i=1}^m \left(\omega^j (-1)^{\I[i\in T_0]} \right)^{x_i} =
\frac{1}{2^m s} \sum_{j=0}^{s-1} \prod_{i=1}^m\left(1 + \omega^j (-1)^{\I[i\in T_0]}\right) \;.
\end{equation*}
Note that the terms with $2j\equiv 0\pmod{s}$ have indices $i$ such that $\omega^j (-1)^{\I[i\in T_0]} = -1$,
and do not contribute to the sum. Other terms will have each value
of $|1 + \omega^j (-1)^{\I[i\in T_0]}|$ at most $2-\Omega(1/s^2)$.
Therefore, $\E_{\bX\sim\mathcal{R}}[\chi_{T_0}(\bX)] = \exp(-\Omega(m/s^2)).$ This completes our proof.
\end{proof}
We also note that $\D_{-}$ is clearly bounded above by $\mathrm{Bin}(m,1/2).$
We define $J$ to be the union of the $d$ elements of $m\cup\{0\}$ congruent to $0$ modulo $s$
that are closest to $m/2$. We note that the measure of $\D_{-}$ outside $J$ is clearly at most
the probability that $\mathrm{Bin}(m,1/2)$ is more than $ds/2$ from $m/2$, which is at most $\zeta$ by standard
tail bounds.

\paragraph{Definition of the Measure $\D_{+}$.}
Intuitively, we would like to define $\D_{+}$ to be equal to some suitable multiple (say, $3$)
of the standard Binomial measure $\mathrm{Bin}(m,1/2)$.
Such a definition would satisfy the desired moment-matching conditions
(property 3 of~\cref{prop:mainProp-disc}) {with zero error}
and would also guarantee that $\D_{+} > 2\D_{-}$ on $\overline{J}$,
as desired (property 1(b)).
However, this candidate definition does not satisfy property 1(a),
i.e., that $\D_{+}$ be equal to $0$ on $J$. To satisfy the latter property, we will need to carefully
modify this measure.
The key lemma is the following:

\begin{lemma} \label{lem:duality-disc}
Let $s\ge\omega(k^4)$. There exists a signed measure $\mu$ on $\{-s+1,-s+2,\ldots,s-1\}$ 
such that: (i) For any integer $0\leq t\leq k$, $\sum_{i=1-s}^{s-1} \mu(i) i^t = 0$, 
(ii) $\mu(0)=-1$, (iii) $|\mu(i)| < 1/10, i\ne0$.
\end{lemma}
\begin{proof}
The conditions on $\mu$ define a linear program (LP).
We will show that this LP is feasible by showing that the dual LP is infeasible.
The dual LP asks for a degree at most $k$ real polynomial $q(x)$ such that
\begin{align*}
|q(0)| \geq (1/11) \sum_{i=1-s}^{s-1} |q(i)| \;.
\end{align*}
Consider the parameterization $p(\theta) = q(s \sin(\theta))$.
We will leverage the fact that $p(\theta)$ is a degree-$k$ polynomial in $e^{\bi\theta}$ and $e^{-\bi\theta}$.
In particular, $p(\theta)$ can be written as
\begin{align*}
p(\theta) = \sum_{j=-k}^k a_j e^{\bi j\theta} \;,
\end{align*}
for some complex coefficients $a_j \in \mathbb{C}$.
By normalizing, we can assume that $\sum_{j=-k}^k |a_j|^2 = 1$.
Then, for any $\theta$, we have that
\begin{align*}
|p(\theta)| \leq \sum_{j=-k}^k |a_j| = O(\sqrt{k}) \;,
\end{align*}
where the final inequality follows from the Cauchy-Schwarz.
In particular, $|q(0)|=|p(0)|=O(\sqrt{k}).$
In addition, for any $\theta$, by Cauchy-Schwarz, we have that
\begin{align*}
|p'(\theta)| = \left| \sum_{j=-k}^k j a_j e^{\bi j \theta}\right| \leq \sum_{j=-k}^k |j| |a_j| \leq \sqrt{ \sum_{j=-k}^k j^2} = O(k^{3/2}) \;.
\end{align*}
Finally, we note that
\begin{align*}
\frac{1}{2\pi} \int_0^{2\pi} |p(\theta)|^2 d\theta = \sum_{j=-k}^k |a_j|^2 = 1 \;.
\end{align*}
Combining the latter  with the fact that $|p(\theta)|=O(\sqrt{k})$, we obtain that
\begin{align*}
\int_0^{2\pi} |p(\theta)|d\theta = \Omega(k^{-1/2}) \;.
\end{align*}
For any $\theta\in [0,2\pi]$, let $n(\theta)$ be the closest $\phi\in [0,2\pi]$ such
that $s \sin(\phi)$ is an integer in $\{1-s,2-s,\ldots,s-1\}$.
It is not hard to see that $|n(\theta)-\theta| = O(s^{-1/2})$ for all such $\theta$.
Furthermore, we have that
\begin{align*}
|p(n(\theta)) - p(\theta)| \leq |n(\theta)-\theta| \, \sup_{\theta'\in[0,2\pi]} |p'(\theta')| \leq O(k^{3/2} s^{-1/2}) \;. 
\end{align*}
We can thus write
\begin{align*}
\Omega(k^{-1/2}) = \int_0^{2\pi} |p(\theta)|d\theta \leq \int_0^{2\pi} |p(n(\theta))|d\theta + O(k^{3/2}s^{-1/2}) \;.
\end{align*}
Therefore,
\begin{align*}
\int_0^{2\pi} |p(n(\theta))|d\theta \geq \Omega(k^{-1/2}) \;.
\end{align*}
On the other hand, each value of $p(n(\theta))$ is equal to the value of $q$ evaluated at
some integer between $1-s$ and $s-1$. Furthermore, it is not hard to see
that each such integer occurs for at most a total of $O(s^{-1/2})$ range of $\theta$'s. Therefore,
we get that
\begin{align*}
O(s^{-1/2}) \sum_{i=1-s}^{s-1} |q(i)| \geq  \Omega(k^{-1/2}) \;.
\end{align*}
Combining with the fact that $|q(0)| = O(k^{1/2})$, this shows that it is impossible that
\begin{align*}
|q(0)| \geq 1/4 \sum_{i=1-s}^{s-1} |q(i)| \;.
\end{align*}
This completes our proof.
\end{proof}

We are now ready to construct the measure $\D_{+}$.
We begin with the measure $3\mathrm{Bin}(m,1/2)$.
We then for each element $x\in J$ take the measure $\mu$ from~\cref{lem:duality-disc},
translate it to center around $x$ and add an appropriate multiple of it to $\D_{+}$ in order to make $\D_{+}(x)=0$.
It is clear that the first $k$ moments of $\D_{+}$ agree with those moments of $3\mathrm{Bin}(m,1/2)$,
and from there it follows that $\D_{+}$ satisfies~\cref{cond:moments-disc} with $\nu=0$, since
for any $0\le t\le k$ and any point $x\in J$, we have that
\begin{align*}
\sum_{i=1-s}^{s-1}\mu(i)(x+i)^t = \sum_{i=1-s}^{s-1}\mu(i) \sum_{\ell=0}^t\binom{t}{\ell}i^\ell x^{t-\ell}
= \sum_{\ell=0}^t\binom{t}{\ell}x^{t-\ell} \sum_{i=1-s}^{s-1}\mu(i)i^\ell=0 \;,
\end{align*}
which means that we never change the moments by making $\D_+(x)=0$.
Therefore, we have $\D_{+}$ is $0$ on $J$ by our construction.
We also claim that $\D_{+}$ is bounded between $2\mathrm{Bin}(m,1/2)$ and $4\mathrm{Bin}(m,1/2)$ on $\bar{J}$.
For this, we note that for any $x\notin J$, there are at most two integers, $x'$ and $x''$,
that are in $J$ and within distance $s$ of $x$.
It is clear that
\begin{align*}
|\D_{+}(x) - 3\mathrm{Bin}(m,1/2)(x)| \leq (3/10)(\mathrm{Bin}(m,1/2)(x')+\mathrm{Bin}(m,1/2)(x'')) \;.
\end{align*}
It suffices to show that
$\frac{\mathrm{Bin}(m,1/2)(x')}{\mathrm{Bin}(m,1/2)(x)} < 3 / 2$
along with the analogous statement for $x''$. 
However, the log of the ratio is easily seen to be $O(s^2 d/m)=o(1)$, which suffices.
This completes the proof of~\cref{prop:mainProp-disc}.
\end{proof}

\subsection{Parameter Setting for the SQ-hard Distributions}\label{ssec:para-setting}
We will consider the following family of hardness distributions which will be used in the proof of all SQ hardness results throughout this article.
Let $C>0$ be a sufficiently large universal constant.
Let $m$ be a positive integer and $m'$ be an integer on the order of $Cm$.
Let $d$ be an integer on the order of $m^{1/10}$,
$s$ an integer on the order of $m^{4/9}$, and $k$ an integer on the order of $m^{2/19}$.
Observe that
$\binom{2d+m'}{m'} \leq (m')^{2d} =\exp(O(C m^{1/10} \log(m)))$.
Select $m$ as large as possible so that the above is less than $M$. Decreasing $M$ if necessary,
we can assume that $M=\binom{2d+m'}{m'}$.
We consider the Veronese mapping $V_{O(d)}:\R^{m'}\rightarrow \R^M$,
such that the coordinate functions of $V_{O(d)}$ are exactly the monomials
in $m'$ variables of degree at most $O(d)$.
We define measures $\D_{+}$ and $\D_{-}$ on $[m]\cup\{0\}$,
as given by~\cref{prop:mainProp-disc},
with $k, s$ and $d$ as above, and taking $\log(1/\zeta)$ a sufficiently small multiple of $(ds)^2/m$,
so that $\zeta = \exp(-\Omega(m^{4/45})) = \exp(-\Omega(\log(M)^{8/9}))$.
It is easily verified that these parameters satisfy the assumptions of~\cref{prop:mainProp-disc}.
For a subset $S\subseteq [m']$ of size $m$ and labels $a\ne b\in\R$, define the distribution $\p_{S,a,b}^{\D_{+},\D_{-},p}$
as in~\cref{def:testing}, with $p=\|\D_{+}\|_1/(\|\D_{+}\|_1 + \|\D_{-} \|_1)$. 
We will consider the distribution $(\bX',Y')$ on $\{0, 1\}^M \times \{a,b\}$ by drawing $(\bX,Y)$ from
$\p_{S,a,b}^{\D_{+},\D_{-},p}$ and letting $\bX'=V_{O(d)}(\bX)$ and $Y'=Y$. 
It is easy to see that finding a hypothesis that predicts $y'$ given $\bx'$ 
is equivalent to finding a hypothesis for $y$ given $\bx$
(since $y=y'$ and there is a known 1-1 mapping between $\bx$ and $\bx'$).
The pointwise bounds on $\D_{+}$ and $\D_{-}$ imply that
$\chi^2(\D_{+}/\|\D_{+}\|_1,\mathrm{Bin}(m,1/2))+\chi^2(\D_{-}/\|\D_{-}\|_1,\mathrm{Bin}(m,1/2))=O(s^2)$.
The parameter $\nu$ in~\cref{prop:gen-sq-prop} is at most
$sm^k\exp(-\Omega(m/s^2))=\exp(-\Omega(m^{1/9}))$.
Note that as $M = \exp(\wt{O}(m^{1/10}))$, this is $\exp(-\Omega(\log(M)^{1.1}))$.
As $k$ is also $\Omega(\log(M)^{1.05})$, 
we have that $\tau = \exp(-\Omega(\log(M)^{1.05}))\le1/\poly(M)$.
In the remaining part of this article, we will use $\bx',\bX',y',Y'$ 
without clarification to denote the results of $\bx,\bX,y,Y$ 
after the Veronese mapping $V_{O(d)}:\R^{m'}\rightarrow \R^M$.

\section{Concrete SQ Hardness Results}\label{sec:SQ-neuron}
\subsection{SQ Hardness of Learning Low-Weight Boolean Halfspaces with Massart Noise} \label{sec:proof-disc}
In this subsection, we prove our SQ hardness results 
for Massart learning of low-weight half-spaces and $\relu$s.

In particular, we start by proving the following theorem.

\begin{theorem}[SQ Hardness for Low-weight Massart Halfspaces on $\{0, 1\}^M$] \label{thm:discrete-ltfs}
Any SQ algorithm that learns $\{\pm 1\}$-weight halfspaces on $\{0,1\}^M$,  
in the presence of Massart noise with $\eta=1/3$, 
to 0-1 error better than $1/\poly(\log(M))$ requires either queries of 
accuracy better than $\tau:=\exp(-\Omega(\log(M)^{1.05}))$ 
or at least $1/\tau$ statistical queries. 
This holds even if the optimal classifier has 0-1 error 
$\exp(-\Omega(\log M)^{8/9})$.
\end{theorem}

\begin{proof}
Our proof will make use of the SQ framework of~\cref{ssec:generic-discrete}
and will crucially rely on the one-dimensional construction of~\cref{prop:mainProp-disc}. 
In this subsection, we fix the labels $a=1,b=-1$, 
and apply the construction in~\cref{ssec:para-setting} 
to obtain the joint distributions $(\bX,Y)$ and $(\bX',Y')$. 
Note that $y=y'$ and there is a known 1-1 mapping between 
$\bx$ and $\bx'$, therefore finding a hypothesis that predicts 
$y'$ given $\bx'$ is equivalent to finding a hypothesis for $y$ given $\bx$.

\begin{claim}\label{clm:pv-Massart-LTF-disc}
The distribution $(\bX',Y')$ over $\{0, 1\}^M \times \{\pm1\}$  is a Massart LTF distribution
with optimal misclassification error $\opt_\adv\le\exp(-\Omega(\log(M)^{8/9}))$ 
and Massart noise rate upper bound of $\eta = 1/3$.
\end{claim}
\begin{proof}
For a $\bv_S$ the vector whose $i^{th}$ coordinate is $1$ if $i\in S$ and $0$ otherwise, 
let $g:\{0,1\}^{m'} \rightarrow \{\pm 1\}$ be defined as
$g(\bx)=-1$ if and only if $\bv_S^T\bx \in J$. 
In this way, we are able to write $g$ as a degree-$2d$ PTF, 
i.e., $g(\bx)=\sign(\prod_{z\in J}(\bv_S^T\bx-z)^2)$.
Therefore, there exists some LTF $L:\R^M\rightarrow \{\pm 1\}$
such that $g(\bx)=L(\bx')=L(V_{2d}(\bx))$ for all $\bx$. 
We now bound the error for LTF $L$ under the distribution $(\bX',Y')$. 
By the law of total probability, we have that
\begin{align*}
&\quad\Pr_{(\bX',Y')}\left[Y'\ne L(\bX')\right]=\Pr_{(\bX,Y)}\left[Y\ne g(\bX)\right]\\&\le\Pr_{(\bX,Y)}[Y\ne g(\bX)\mid Y=1]+\Pr_{(\bX,Y)}[Y\ne g(\bX)\mid Y=-1] \;.
\end{align*}
We note that our hard distribution returns $(\bx',y')$ with $y'=L(\bx')$,
unless it picked a sample corresponding to a sample of $\D_{-}$ coming from $\overline{J}$, therefore,
\begin{align*}
\Pr_{(\bX',Y')}\left[Y'\ne L(\bX')\right]\le\Pr_{(\bX,Y)}[Y\ne g(\bX)\mid Y=-1]\le\zeta \;,
\end{align*}
which implies that $\opt_\adv\le\zeta\le\exp(-\Omega(\log(M)^{8/9}))$.
We then show that $(\bX',Y')$ is a Massart LTF distribution 
with noise rate upper bound of $\eta=1/3$. For any fixed $\bx'\in\R^M$, we have that
\begin{align*}
&\quad\frac{\Pr_{(\bX',Y')}[Y'=1\mid \bX'=\bx']}{\Pr_{(\bX',Y')}[Y'=-1\mid \bX'=\bx']} 
=\frac{\Pr_{(\bX,Y)}[Y=1\mid \bX=\bx]}{\Pr_{(\bX,Y)}[Y=-1\mid \bX=\bx]}\\
&=\frac{\Pr_{(\bX,Y)}[Y=1]\cdot\Pr_{(\bX,Y)}[\bX=\bx\mid Y=1]}{\Pr_{(\bX,Y)}[Y=-1]\cdot\Pr_{(\bX,Y)}[\bX=\bx\mid Y=-1]}
=\frac{\|\D_+\|_1\cdot\p_S^{\D_+}(\bx)}{\|\D_-\|_1\cdot\p_S^{\D_-}(\bx)}
=\frac{\D_+(\bv_S^T\bx)}{\D_-(\bv_S^T\bx)} \;.
\end{align*}
Therefore, if $\bv_S^T\bx\in J$, the above ratio will be 0 and $L(\bx')=-1$, 
which means that the noise rate $\eta(\bx')=0$; 
otherwise, the above ratio will be at least $2$ 
(since $\D_+>2\D_-$ on $\bar{J}$ by property 1(b) of~\cref{prop:mainProp-disc}) 
and $L(\bx')=1$, which means that $\eta(\bx')\le1/3$.
This completes the proof of the claim.
\end{proof}

We now show that the $(\D_+,\D_-,1,-1,m')$-Hidden Junta Testing Problem 
efficiently reduces to our learning task. In more detail, 
we show that any SQ algorithm that computes a hypothesis $h'$ satisfying
$\pr_{(\bX',Y')}[h'(\bX')\ne Y'] < \min(p,1-p)-2\sqrt{2\tau}$ can be used as a black-box
to distinguish between $\p^{\D_+,\D_-,p}_{S,a,b}$, 
for some unknown subset $S\subseteq[m]$ with $|S|=m$, and $U^p_{m'}$.
Since there is a 1-1 mapping between $\bx\in\{0,1\}^{m'}$ and $\bx'\in\{0,1\}^M$,
we denote $h:\{0,1\}^{m'}\mapsto\{\pm1\}$ to be $h(\bx)=h'(\bx')$.
We note that we can
(with one additional {query} to estimate the $\pr[h'(\bX')\ne Y']$ within error $\sqrt{2\tau}$)
distinguish between (i) the distribution $\p^{\D_+,\D_-,p}_{S,a,b}$,
and (ii) the distribution  $U^p_{m'}$.
This is because for any $h$ we have that
$\pr_{(\bX,Y)\sim U^p_{m'}}[h(\bX)\ne Y] \geq \min(p,1-p)$.
Applying~\cref{prop:gen-sq-prop}, we determine 
that any SQ algorithm which, given access to a distribution $\p$
so that either $\p=U_{m'}^p$, or $\p$ is given by $\p_{S,a,b}^{\D_+,\D_-,p}$
for some unknown subset $S\subseteq[m']$ with $|S|=m$,
correctly distinguishes between these two cases 
with probability at least $2/3$ must either make queries 
of accuracy better than $\sqrt{2\tau}$
or must make at least
$2^{\Omega(m)}\tau/(\chi^2(A, \mathrm{Bin}(m,1/2))+\chi^2(B, \mathrm{Bin}(m,1/2)))$
statistical queries. Therefore, it is impossible for an SQ algorithm
to learn a hypothesis with error better than 
$\min(p,1-p)-2\sqrt{2\tau}=\Theta(1/s)-\Theta(\sqrt{\tau})=1/\mathrm{polylog}(M)$
without either using queries of accuracy better than $\tau$
or making at least $2^{\Omega(m)}\tau/\mathrm{polylog}(M) > 1/\tau$ many queries.
This completes the proof of the SQ-hardness.

It remains to argue that the underlying halfspaces 
in the hard instance can be assumed to have
$\{\pm 1\}$ weights. To deal with the weights, we note that $g$ is a degree-$2d$ PTF
that can be defined as the product of $2d$ linear polynomials $L_i$,
so that each $L_i$ has integer coefficients 
and the sum of the absolute values of these coefficients is $O(m)$.
This means that $g$ can be defined by a degree-$2d$ polynomial
with integer coefficients and the sum of whose absolute 
values is at most $O(m)^{2d} = \poly(M)$. By doubling these coefficients, 
we can assume that they are all even.
Therefore, the linear threshold function $L$ can be defined 
by a linear polynomial with even integer weights
each of which has absolute value at most $W$. 
If we replace our distribution over $\{0,1\}^M$
by a distribution over $\{0,1\}^{MW}$ by duplicating each coordinate $W$ times
(i.e., creating a new distribution with coordinates $z_{i,j}$ 
for $i\in [M]$ and $j\in [W]$ with $z_{i,j}=x_i$ for all $i,j$),
we can rewrite $L(x)$ as an LTF $L'(z)$, where $L'$ has $\{\pm 1\}$-weights.
This is done by replacing a term $a_i x_i$ by 
$\sum_{j=1}^{(a_i+W)/2} z_{i,j}-\sum_{j=(a_i+W)/2+1}^W z_{i,j}$.
This completes the proof of~\cref{thm:discrete-ltfs}.
\end{proof}

\subsection{SQ Hardness of Learning a Single Neuron with Massart Noise}
In this subsection, we prove our SQ hardness result of learning a single neuron with ReLU activation 
and Massart noise.
The standard $\relu$ function is defined by $\relu(t)=\max(t,0),\forall t\in\R$. 
For technical convenience, we will consider the following linear transformation of the standard $\relu$,
$ \wh{\relu}(t)= -1$ if  $t<0$, and $ \wh{\relu}(t)= -1+2t$ otherwise.
We note that our SQ hardness result for the $\wh{\relu}$ function 
applies to the standard $\relu$ function as well.

\begin{theorem}[SQ Hardness of Massart Learning ReLUs]\label{thm:sq-relu}
Any SQ algorithm that learns a single neuron with ReLU activation on $\R^M$, 
in the presence of Massart noise with $\eta=1/3$, 
within squared error better than $1/\poly(\log(M))$ 
requires either queries of accuracy better than 
$\tau:=\exp(-\Omega(\log(M)^{1.05}))$ or at least $1/\tau$ statistical queries.
This holds even if (i) the optimal neuron has squared error $\exp(-\Omega(\log M)^{8/9})$, 
(ii) The $\bX$ values are supported on $\{0,1\}^M$, and (iii) the total weight of the neuron is $\poly(M)$.
\end{theorem}

\noindent Throughout this subsection, we need the following technical lemma.
\begin{lemma}\label{lem:poly-interpolation}
Let $J$ be a union of $d$ points in $[m]\cup\{0\}$ for some odd integer $d$. 
Then there exists a real univariate polynomial $p(x)$ of degree $O(d)$ 
such that $p(x)=1,\forall x\in J$, and $p(x)\le0,\forall x\in\bar{J}$. 
In addition, the absolute value of the coefficients of $p(x)$ is at most $m^{O(d)}=\poly(M)$.
\end{lemma}

\begin{proof}
Let $J=\{x_1,\ldots,x_d\}$. Define $q(x)=-\prod_{i=1}^{d}(x-(x_i-1/2))(x-(x_i+1/2))$.
By definition, we have that $q(x)>0,\forall x\in J$, and $q(x)<0,\forall x\in\bar{J}$.
Then, by polynomial interpolation, there exists a real univariate polynomial $r$ 
of degree $d-1$ such that $r(x_i)=\frac{1}{\sqrt{q(x_i)}},1\le i\le d$.
Consider the real univariate polynomial $p(x)=r^2(x)q(x)$. For any $1\le i\le d$,
we have that $p(x_i)=r^2(x_i)q(x_i)=1$ and for any $x\in\bar{J}$, 
we have that $p(x)\le0$ since $q(x)<0,\forall x\in\bar{J}$.
Finally by polynomial interpolation, we know that the 
absolute value of every coefficient of $r(x), p(x)$ is at most $m^{O(d)}=\poly(M)$.
\end{proof}

\begin{proof}[Proof of~\cref{thm:sq-relu}]
Our proof will make use of the SQ framework of~\cref{ssec:generic-discrete}
and will crucially rely on the one-dimensional construction of~\cref{prop:mainProp-disc}. 
In this section, we fix the labels $a=-1, b=1$, and apply 
the construction in~\cref{ssec:para-setting} to obtain 
the joint distributions $(\bX,Y)$ and $(\bX',Y')$. 
Note that $y=y'$ and there is a known 1-1 mapping 
between $\bx$ and $\bx'$, therefore finding a hypothesis 
that predicts $y'$ given $\bx'$ is equivalent to finding a hypothesis for $y$ given $\bx$.

\begin{claim}\label{clm:pv-Massart-relu}
The distribution $(\bX',Y')$ over $\{0, 1\}^M \times \{\pm1\}$ 
is a Massart single neuron distribution with ReLU activation, 
with optimal squared error $\opt_\adv\le\exp(-\Omega(\log(M)^{8/9}))$ 
and Massart noise rate upper bound of $\eta = 1/3$.
\end{claim}
\begin{proof}
Let $\bv_S$ be the vector whose $i^{th}$ coordinate 
is $1$ if $i\in S$ and $0$ otherwise.
By~\cref{lem:poly-interpolation}, there is a real 
univariate polynomial $p$ of degree $O(d)$ 
such that $p(\bv_S^T\bx)=1,\bv_S^T\bx\in J$ and
$p(\bv_S^T\bx)\le0,\bv_S^T\bx\notin J$.
Let $g(\bx):=\wh{\relu}(p(\bv_S^T\bx))$.
Since the absolute value of every coefficient of $p$ is at most $m^{O(d)}=\poly(M)$,
by our definition, the total weight of the corresponding neuron $g$ is at most $m^{O(d)}=\poly(M)$.
Therefore, there exists some $\wh{\relu}$ function $L:\R^M\rightarrow\R$
such that $g(\bx)=L(\bx')=L(V_{O(d)}(\bx))$ for all $\bx$.
We now bound the error for $L$ under the distribution $(\bX',Y')$. 
By the law of total expectation, we have that
\begin{align*}
&\quad\E_{(\bX',Y')}\left[(Y'-L(\bX'))^2\right]= 
\E_{(\bX,Y)}\left[(Y-g(\bX))^2\right]\\
&\le\E_{(\bX,Y)}\left[(Y-g(\bX))^2\mid Y=1\right]+\E_{(\bX,Y)}\left[(Y-g(\bX))^2\mid Y=-1\right]\;.
\end{align*}
We note that our hard distribution returns $(\bX',Y')$ with $Y'=L(\bX')$,
unless it picked a sample corresponding to a sample 
of $\D_{-}$ coming from $\overline{J}$, therefore,
\begin{align*}
\E_{(\bX',Y')}\left[(Y'-L(\bX'))^2\right]\le\E_{(\bX,Y)}\left[(Y-g(\bX))^2\mid Y=1\right]\le4\zeta \;,
\end{align*}
which implies that $\opt_\adv\le4\zeta\le\exp(-\Omega(\log(M)^{8/9}))$. 
We then show that $(\bX',Y')$ is a Massart single neuron distribution 
with $\wh{\relu}$ activation and with noise rate upper bound of $\eta=1/3$. 
For any fixed $\bx'\in\R^M$, we have that
\begin{align*}
&\quad\frac{\Pr_{(\bX',Y')}[Y'=-1\mid \bX'=\bx']}{\Pr_{(\bX',Y')}[Y'=1\mid \bX'=\bx']}=\frac{\Pr_{(\bX,Y)}[Y=-1\mid \bX=\bx]}{\Pr_{(\bX,Y)}[Y=1\mid \bX=\bx]}\\
&=\frac{\Pr_{(\bX,Y)}[Y=-1]\cdot\Pr_{(\bX,Y)}[\bX=\bx\mid Y=-1]}{\Pr_{(\bX,Y)}[Y=1]\cdot\Pr_{(\bX,Y)}[\bX=\bx\mid Y=1]} 
=\frac{\|\D_+\|_1\cdot\p_S^{\D_+}(\bx)}{\|\D_-\|_1\cdot\p_S^{\D_-}(\bx)}=\frac{\D_+(\bv_S^T\bx)}{\D_-(\bv_S^T\bx)} \;.
\end{align*}
Therefore, if $\bv_S^T\bx\in J$, the above ratio will be 0 and $L(\bx')=-1$,
which means that the noise rate $\eta(\bx')=0$;
otherwise the above ratio will be at least $2$ 
(since $\D_+>2\D_-$ on $\bar{J}$ by property 1(b) of~\cref{prop:mainProp-disc}) 
and $L(\bx')=1$, which means that $\eta(\bx')\le1/3$.
This completes the proof of the claim.
\end{proof}

We now show that the $(\D_+,\D_-,-1,1,m')$-Hidden Junta Testing Problem 
efficiently reduces to our learning task.
In more detail, we show that any SQ algorithm that computes a hypothesis $h'$ satisfying
$\E_{(\bX',Y')}[(h'(\bX')-Y')^2] < 4p-4p^2-2\sqrt{2\tau}$ can be used as a black-box
to distinguish between $\p^{\D_+,\D_-,p}_{S,a,b}$, 
for some unknown subset $S\subseteq[m']$ with $|S|=m$, and $U^p_{m'}$.
Since there is a 1-1 mapping between $\bx\in\{0,1\}^{m'}$ and $\bx'\in\{0,1\}^M$,
we denote $h:\{0,1\}^{m'}\mapsto\R$ to be $h(\bx)=h'(\bx')$.
We note that we can
(with one additional {query} to estimate the $\E[(h'(\bX')-Y')^2]$ within error $\sqrt{2\tau}$)
distinguish between (i) the distribution $\p^{\D_+,\D_-,p}_{S,a,b}$,
and (ii) the distribution  $U^p_{m'}$.
This is because for any $h$ we have that
\begin{align*}
\E_{(\bX,Y)\sim U^p_{m'}}&[(h(\bX)-Y)^2]
=1-2(1-2p)\E_{(\bX,Y)\sim U^p_{m'}}[h(\bX)]+
\E_{(\bX,Y)\sim U^p_{m'}}[h(\bX)^2]\\&\ge1-2(1-2p)\E_{(\bX,Y)\sim U^p_{m'}}[h(\bX)]
+\E_{(\bX,Y)\sim U^p_{m'}}[h(\bX)]^2 \ge4p-4p^2 \;.
\end{align*}
Applying~\cref{prop:gen-sq-prop}, we determine 
that any SQ algorithm which, given access to a distribution $\p$
so that either $\p=U_{m'}^p$, or $\p$ is given by $\p_{S,a,b}^{\D_+,\D_-,p}$
for some unknown subset $S\subseteq[m']$ with $|S|=m$,
correctly distinguishes between these two cases 
with probability at least $2/3$ must either make queries of accuracy better than $\sqrt{2\tau}$
or must make at least
$2^{\Omega(m)}\tau/(\chi^2(A, \mathrm{Bin}(m,1/2))+\chi^2(B, \mathrm{Bin}(m,1/2)))$
statistical queries. Therefore, it is impossible for an SQ algorithm
to learn a hypothesis with error better than 
$4p-4p^2-2\sqrt{2\tau}=\Theta(1/s)-\Theta(\sqrt{\tau})=1/\mathrm{polylog}(M)$
without either using queries of accuracy better than $\tau$
or making at least $2^{\Omega(m)}\tau/\mathrm{polylog}(M) > 1/\tau$ many queries.
This completes the proof of~\cref{thm:sq-relu}.
\end{proof}

\subsection{SQ Hardness of Learning a Single Neuron with $L_2$-Massart Noise}\label{sec:Daniel}
In this section, we prove our SQ hardness result of learning 
a single neuron with fast convergent activations and $L_2$-Massart noise.
Without loss of generality, we consider activations which converge on the negative side. 
For such an activation $f$, let $f_-:=f(-\infty)$ and $c_+$ be a constant such that $f(c_+)\ne f_-$.
The main theorem of this section is the following.
\begin{theorem}[SQ Hardness of $L_2$-Massart Learning]\label{thm:sq-fast-convergent-daniel}
Let $f:\R\to\R$ be a fast convergent activation. Any SQ algorithm that learns a single neuron
with activation $f$ on $\R^M$, in the presence of $\eta$-$L_2$-Massart noise 
with $\eta=\frac{2(f(c_+)-f_-)^2}{9}$, to squared error better than $1/\poly(\log(M))$ 
requires either queries of accuracy better than $\tau:=\exp(-\Omega(\log(M)^{1.05}))$
or at least $1/\tau$ statistical queries. This holds even if:
\begin{enumerate}[leftmargin=*]
\item The optimal neuron has squared error $\opt_{\daniel}\leq \exp(-\Omega(\log(M)^{8/9}))$,
\item The $\bX$ values are supported on $\{0,1\}^M$, and
\item The total weight of the neuron is $\poly(M)$.
\end{enumerate} 
\end{theorem}
\begin{proof}
Our proof will make use of the SQ framework of~\cref{ssec:generic-discrete}
and will crucially rely on the one-dimensional construction of~\cref{prop:mainProp-disc}. 
In this section, we fix the labels $a=f_-, b=f(c_+)$, and apply the construction in~\cref{ssec:para-setting} 
to obtain the joint distributions $(\bX,Y)$ and $(\bX',Y')$. 
Note that $y=y'$ and there is a known 1-1 mapping between $\bx$ and $\bx'$, 
therefore finding a hypothesis that predicts $y'$ given $\bx'$ 
is equivalent to finding a hypothesis for $y$ given $\bx$.

\begin{claim}
The distribution $(\bX',Y')$ on $\{0, 1\}^M \times \{f_-,f(c_+)\}$ is an 
$L_2$-Massart single neuron distribution with respect to activation $f$,
it has optimal squared error $\opt_{\daniel}\leq \exp(-\Omega(\log(M)^{8/9}))$ 
and $L_2$-Massart noise rate upper bound of $\eta=\frac{2(f(c_+)-f_-)^2}{9}$.
\end{claim}

\begin{proof}
We assume $M>|c_+|$ to be sufficiently large.
Let $\bv_S$ be the vector whose $i^{th}$ coordinate is $1$ if $i\in S$ and $0$ otherwise.
By~\cref{lem:poly-interpolation}, there is a real univariate polynomial 
$q(x)$ of degree $O(d)$ such that $q(x)=1,\forall x\in J$ and $q(x)\le0,\forall x\in\bar{J}$.
Let $p(x)=(c_++M)q(x)-M$ and $g(\bx)=f(p(\bv_S^T\bx))$.
By definition, we have that $p(x)=c_+$ for $x\in J$ and $p(x)\leq -M$ for $x\in\bar{J}$.
Since the absolute value of every coefficient of $p$ 
is at most $m^{O(d)}=\poly(M)$, the weight of the corresponding neuron $g$ is at most $m^{O(d)}=\poly(M)$.
Therefore, there exists some fast convergent activation $L:\R^M\rightarrow\R$
such that $g(\bx)=L(\bx')=L(V_{O(d)}(\bx))$ for all $\bx$.
We now bound the error for $L$ under the distribution $(\bX',Y')$.
We note that conditional on $Y=f_-$, we will always have 
that $\bv_S^T\bx\notin J$ and conditional on $Y=f(c_+)$, 
we will have that $\bv_S^T\bx\notin J$ with probability at most $\zeta$. 
Therefore, by the law of total expectation, we have that
\begin{align*}
&\quad\E_{(\bX',Y')}[(Y'-L(\bX))^2]=\E_{(\bX,Y)}[(Y-g(\bX))^2]
\\
&\le\E_{(\bX,Y)}[(Y-g(\bX))^2\mid Y=f_-]+\E_{(\bX,Y)}[(Y-g(\bX))^2\mid Y=f(c_+)]\\
&\le\E_{(\bX,Y)}[(f_--g(\bX))^2\mid Y=f_-]+2\zeta\E_{(\bX,Y)}[(f_--f(c_+))^2+(f_--g(\bX))^2\mid \bv_S^T\bX\notin J, Y=f(c_+)]\\
&\le1/\poly(M)+2\zeta\cdot(1/\poly(M)+(f_--f(c_+))^2)\\&\le\exp(-\Omega(\log(M)^{8/9}))+\exp(-\Omega(\log(M)^{8/9}))\cdot(1/\poly(M)+(f_--f(c_+))^2)\\&\le\exp(-\Omega(\log(M)^{8/9})) \;,
\end{align*}
where the third inequality follows from the definition of fast convergent activation. 
Therefore, we have that $\opt_\daniel\le\exp(-\Omega(\log(M)^{8/9}))$.
We then show that $(\bX',Y')$ is a $L_2$-Massart 
single neuron distribution with activation $f$ and 
with noise rate upper bound of $\eta=\frac{2(f(c_+)-f_-)^2}{9}$.
Note that for any $\bx\in\R^{m'}$, if $\bv_S^T\bx\in J$, then 
$g(\bx)=f(p(\bv_S^T\bx))=f(c_+)$ and $Y$ will always be $f(c_+)$, 
which implies that the error will always be $0$.
Hence, we assume that $\bv_S^T\bx\notin J$ and have that
\begin{align*}
\frac{\Pr_{(\bX,Y)}[Y=f_-\mid \bX=\bx]}{\Pr_{(\bX,Y)}[Y=f(c_+)\mid \bX=\bx]}
&=\frac{\Pr_{(\bX,Y)}[Y=f_-]\cdot\Pr_{(\bX,Y)}[\bX=\bx\mid Y=f_-]}{\Pr_{(\bX,Y)}[Y=f(c_+)]\cdot\Pr_{(\bX,Y)}[\bX=\bx\mid Y=f(c_+)]}\\
&=\frac{\|\D_+\|_1\cdot\p_S^{\D_+}(\bx)}{\|\D_-\|_1\cdot\p_S^{\D_-}(\bx)}=\frac{\D_+(\bv_S^T\bx)}{\D_-(\bv_S^T\bx)}\ge2 \;,
\end{align*}
which implies that $\Pr_{(\bX,Y)}[Y=f(c_+)\mid \bX=\bx]\le1/3$. Therefore,
\begin{align*}
&\quad\E_{(\bX',Y')}[(Y'-L(\bX'))^2\mid\bX'=\bx']
=\E_{(\bX,Y)}[(Y-g(\bX))^2\mid\bX=\bx]\\&=(f(c_+)-g(\bx))^2\Pr_{(\bX,Y)}[Y=f(c_+)\mid\bX=\bx]+(f_--g(\bx))^2\Pr_{(\bX,Y)}[Y=f_-\mid\bX=\bx]\\
&\le\frac{(f(c_+)-g(\bx))^2}{3}+(f_--g(\bx))^2\le\frac{2\left((f(c_+)-f_-)^2+(f_--g(\bx))^2\right)}{3}+(f_--g(\bx))^2\\&\le\frac{2(f(c_+)-f_-)^2}{3}+1/\poly(M)\le\frac{8(f(c_+)-f_-)^2}{9} \;,
\end{align*}
where the third inequality follows from $\bv_S^T\bx\notin J$ 
and the definition of fast convergent activation.
This completes the proof of the claim.
\end{proof}

We now show that the $(\D_+,\D_-,f_-,f(c_+),m')$-Hidden Junta Testing Problem 
efficiently reduces to our learning task.
In more detail, we show that any SQ algorithm that computes a hypothesis $h'$ satisfying
$\E_{(\bX',Y')}[(h'(\bX')-Y')^2] < p(1-p)(f_--f(c_+))^2-2\sqrt{2\tau}$ can be used as a black-box
to distinguish between $\p^{\D_+,\D_-,p}_{S,a,b}$, 
for some unknown subset $S\subseteq[m']$ with $|S|=m$, and $U^p_{m'}$.
Since there is a 1-1 mapping between $\bx\in\{0,1\}^{m'}$ and $\bx'\in\{0,1\}^M$,
we denote $h:\{0,1\}^{m'}\mapsto\R$ to be $h(\bx)=h'(\bx')$.
We note that we can
(with one additional {query} to estimate the $\E[(h'(\bX')-Y')^2]$ within error $\sqrt{2\tau}$)
distinguish between (i) the distribution $\p^{\D_+,\D_-,p}_{S,a,b}$,
and (ii) the distribution  $U^p_{m'}$.
This is because for any $h$ we have that
\begin{align*}
\E_{(\bX,Y)\sim U^p_{m'}}[(h(\bX)-Y)^2]
&=\E_{(\bX,Y)\sim U^p_{m'}}[h(\bX)^2]-2\E_{(\bX,Y)\sim U^p_{m'}}[h(\bX)]\E_{(\bX,Y)\sim U^p_{m'}}[Y]\\
&\quad+\E_{(\bX,Y)\sim U^p_{m'}}[Y^2]\\
&\ge\E_{(\bX,Y)\sim U^p_{m'}}[h(\bX)]^2-2\E_{(\bX,Y)\sim U^p_{m'}}[h(\bX)]\E_{(\bX,Y)\sim U^p_{m'}}[Y]\\&\quad+\E_{(\bX,Y)\sim U^p_{m'}}[Y^2]\\
&\ge\E_{(\bX,Y)\sim U^p_{m'}}[Y^2]-\E_{(\bX,Y)\sim U^p_{m'}}[Y]^2=p(1-p)(f_--f(c_+))^2 \;.
\end{align*}
Applying~\cref{prop:gen-sq-prop}, we determine that any 
SQ algorithm which, given access to a distribution $\p$
so that either $\p=U_{m'}^p$, or $\p$ is given by $\p_{S,a,b}^{\D_+,\D_-,p}$
for some unknown subset $S\subseteq[m']$ with $|S|=m$,
correctly distinguishes between these two cases 
with probability at least $2/3$ must either 
make queries of accuracy better than $\sqrt{2\tau}$
or must make at least
$2^{\Omega(m)}\tau/(\chi^2(A, \mathrm{Bin}(m,1/2))+\chi^2(B, \mathrm{Bin}(m,1/2)))$
statistical queries. Therefore, it is impossible for an SQ algorithm
to learn a hypothesis with error better than 
$p(1-p)(f_--f(c_+))^2-2\sqrt{2\tau}=\Theta(1/s)-\Theta(\sqrt{\tau})=1/\mathrm{polylog}(M)$
without either using queries of accuracy better than $\tau$
or making at least $2^{\Omega(m)}\tau/\mathrm{polylog}(M) > 1/\tau$ many queries.
This completes the proof of~\cref{thm:sq-fast-convergent-daniel}.
\end{proof}

\section{Conclusion and Future Directions} \label{sec:conclusion}
In this work, we showed that no efficient SQ algorithm can approximate 
the optimal error within any constant factor for learning single neurons with Massart noise.
In the process, we constructed new moment-matching distributions corresponding 
to Boolean halfspaces with Massart noise, which is a result of independent interest.
Importantly, our construction has some additional desirable properties 
which allows us to establish hardness for learning low-weight LTFs,
strengthening the result of~\cite{DK21-SQ-Massart}.
In addition, we provide a simple technique for transforming 
our binary construction into hardness of learning real-valued single neurons with Massart noise.

A number of avenues for future work remain, some of which we briefly discuss below.
Recent work~\cite{DKKTZ22} studied the problem 
of learning halfspaces under the Gaussian distribution with 
Massart noise for $\eta=1/2$.
It is plausible that the $\eta=1/2$ case in our distribution-independent setting 
is much harder than the $\eta=0.49$ case. Establishing such a statement is left
as an interesting open question. Another direction concerns the distribution-specific setting.
Are there efficient algorithms with non-trivial error guarantees 
(e.g., achieving a constant factor approximation) 
for learning single neurons under simple discrete distributions 
(e.g., under the uniform distribution on the hypercube)?

\bibliographystyle{alpha}

\bibliography{allrefs}

\end{document}